%% file: main.tex
\documentclass[letterpaper]{article} % DO NOT CHANGE THIS
\usepackage{aaai25}  % DO NOT CHANGE THIS
\usepackage{times}  % DO NOT CHANGE THIS
\usepackage{helvet}  % DO NOT CHANGE THIS
\usepackage{courier}  % DO NOT CHANGE THIS
\usepackage[hyphens]{url}  % DO NOT CHANGE THIS
\usepackage{graphicx} % DO NOT CHANGE THIS
\usepackage{natbib}  % DO NOT CHANGE THIS AND DO NOT ADD ANY OPTIONS TO IT
\usepackage{caption} % DO NOT CHANGE THIS AND DO NOT ADD ANY OPTIONS TO IT
\usepackage[ruled, noend, noline]{algorithm2e}
\SetKw{Continue}{continue}
\SetKw{Break}{break}

\usepackage{mathtools}
\usepackage{amssymb}
\usepackage{multirow}
\usepackage{booktabs}
\usepackage[T1]{fontenc}
\usepackage{bbold}
\usepackage{amsmath}
\DeclareMathOperator*{\argmin}{argmin}
\usepackage{accents}
\newcommand{\ubar}[1]{\underaccent{\bar}{#1}}
\usepackage{subcaption}
\usepackage{amsthm}
\usepackage{enumitem}
\newtheorem{theorem}{Theorem}
\newtheorem{corollary}{Corollary}
\newtheoremstyle{exampleStyle}  % Custom style name
    {\topsep}    % Space above
    {\topsep}    % Space below
    {}           % Body font (empty means default)
    {}           % Indent amount
    {\itshape}   % Theorem head font
    {.}          % Punctuation after theorem head
    { }          % Space after theorem head
    {#1} % Theorem head spec

\theoremstyle{exampleStyle}
\newtheorem{example}{Example}

\title{Optimal Classification Trees for Continuous Feature Data\\Using Dynamic Programming with Branch-and-Bound}
\author{
    Catalin E. Brita\textsuperscript{\rm 1, \rm 2}, 
    Jacobus G. M. van der Linden\textsuperscript{\rm 2}\thanks{Corresponding author},
    Emir Demirovi\'{c}\textsuperscript{\rm 2}
}
\affiliations{
    \textsuperscript{\rm 1} University of Amsterdam, The Netherlands, \\
    \textsuperscript{\rm 2} Delft University of Technology, The Netherlands\\
    Catalin.Brita@student.uva.nl, \{J.G.M.vanderLinden, E.Demirovic\}@tudelft.nl
}

\begin{document}

\maketitle

%%%%%%% ABSTRACT %%%%%%%%%%%%%%%%%%%%
\begin{abstract}
Computing an optimal classification tree that provably maximizes training performance within a given size limit, is NP-hard, and in practice, most state-of-the-art methods do not scale beyond computing optimal trees of depth three. 
Therefore, most methods rely on a coarse binarization of continuous features to maintain scalability.
We propose a novel algorithm that optimizes trees directly on the continuous feature data using dynamic programming with branch-and-bound. We develop new pruning techniques that eliminate many sub-optimal splits in the search when similar to previously computed splits and we provide an efficient subroutine for computing optimal depth-two trees. 
Our experiments demonstrate that these techniques improve runtime by one or more orders of magnitude over state-of-the-art optimal methods and improve test accuracy by 5\% over greedy heuristics.
\end{abstract}

%%%%%%%% INTRODUCTION %%%%%%%%%%%%%%%

\section{Introduction}
\label{sec:introduction}
Decision trees combine human comprehensibility and accuracy and can capture complex non-linear relationships.
As such, decision trees are well-suited models for explainable AI \citep{rudin2019stop_black_box, arrieta2020xai}.
Despite their straightforwardness, constructing an optimal decision tree (ODT, a tree with the smallest training error) of limited size is NP-hard \citep{laurent1976constructing}.
Therefore, greedy heuristics, such as CART \citep{breiman1984cart}
and C4.5 \citep{quinlan1993c4.5}, have been widely used.
These methods attain scalability by locally optimizing an information gain metric at each decision node, but yield on average less accurate and larger trees than optimal \citep{linden2024opt_vs_greedy}.

To compute ODTs, some employ general-purpose solvers such as mixed-integer programming (MIP) \citep{bertsimas2017optimal}, constraint programming \citep{verhaeghe2020cp_odt}, or Boolean satisfiability (SAT) solvers \citep{narodytska2018sat_odt}. However, these approaches struggle to scale with the number of observations and features.

Better scalability is obtained by specialized algorithms using dynamic programming (DP) and branch and bound (BnB) \citep{aglin2020learning, demirovic2022murtree}.
However, most of these algorithms cannot directly deal with numeric data which are frequently present in real-world datasets. Therefore, these algorithms either use a coarse binarization resulting in loss of optimality; or require a binary feature for every possible threshold on the numeric data, which drastically hurts scalability, because their runtime scales exponentially with the number of such features.

To the best of our knowledge, Quant-BnB \citep{mazumder2022odt_continuous_bnb} is the only specialized optimal algorithm that processes continuous features directly. This BnB algorithm considers splits at certain quantiles of the feature distribution. Obtained solutions are used as bounds for pruning, whereas other parts are further explored with splits on quantiles of subregions of the feature distribution.
Though Quant-BnB outperforms other algorithms by a large margin on numeric data, scalability is still an issue: it requires several hours to find trees of depth three for some datasets and the authors recommend against using it beyond depth three.

In this work, we present ConTree, which combines existing DP and BnB techniques with new bounding techniques to improve the scalability of computing optimal classification trees for numeric data. Our new lower-bounding techniques prune large parts of the search space while adding negligible computational overhead. 
Furthermore, for depth-two trees, we propose a specialized subroutine that exploits the fact that we can sort numeric data.

Our experiments show that these algorithmic improvements boost scalability by one or more orders of magnitude over Quant-BnB and three previous MIP and SAT approaches by an ever larger margin.
This makes ConTree the first approach to compute depth-four ODTs beyond small or binarized datasets within a reasonable time.
When trained with the same size limit, ConTree's test accuracy is on average 5\% higher than CART and 0.7\% higher than ODTs trained with a coarse binarization.

%%%%%%%%%% Related Work %%%%%%%%%%%%%%%

\section{Related Work}
\label{sec:literature_review}
\paragraph{Heuristics} Traditionally, decision trees are trained using top-down induction heuristics such as CART \citep{breiman1984cart} and C4.5 \citep{quinlan1993c4.5} because of their scalability.
These heuristics recursively divide the data based on local criteria such as information gain or Gini impurity. On average this yields trees that are larger than the optimal tree \citep{murthy1995dt_greedy_evaluation} or, if constrained by a fixed depth, have lower out-of-sample accuracy than optimal trees under the same size limit \citep{linden2024opt_vs_greedy}.

\paragraph{Optimal} Optimal decision trees globally optimize an objective (e.g., minimize the number of misclassifications) within a given size limit on the training data.
Computing ODTs, however, is NP-hard \citep{laurent1976constructing}, and thus scalability is challenging. To address this challenge, many different approaches have been proposed. 

\paragraph{General-purpose solvers}
The first MIP approaches were proposed by \citet{bertsimas2017optimal} and \citet{verwer2017flexible}.
Many other formulations followed, typically using binarization of the continuous feature data to improve the scalability \citep{verwer2019learning, gunluk2021categorical, aghaei2024strong, liu2024oct_leaf_bin}.
\citet{hua2022odt_mip_branching} instead use a multi-stage MIP model with novel lower bounds.
\citet{ales2024new_mip_odt} obtain stronger linear relaxations from a novel quadratic formulation.
\citet{narodytska2018sat_odt}, \citet{janota2020dt_sat_expl_paths}, \citet{avellaneda2020efficient_sat_odt}, \citet{ hu2020maxsat_odt} and \citet{alos2023maxsat_odt} propose SAT models that also require binarization of continuous feature data.
As far as we know, \citet{m_shati2021sat_dt_nonbin, shati2023sat_odt_nonbin} propose the only SAT-based algorithm that can directly process continuous and categorical features.
Finally, \citet{verhaeghe2020cp_odt} propose a constraint programming approach that also requires binary data.
However, despite improvements, these MIP, SAT, and CP methods still face problems scaling beyond a few thousand data instances and trees of depth three.

\paragraph{Specialized algorithms}
\citet{nijssen2007mining, nijssen2010dl8_constraints} introduced DL8, an early DP approach.
\citet{aglin2020learning, aglin2020pydl8} improved it to DL8.5 with branch-and-bound and extended caching techniques.
\citet{hu2019sparse} and \citet{lin2020generalized_sparse} contribute new lower bounds including a subproblem similarity bound.
\citet{demirovic2022murtree} introduce a specialized subroutine for trees of depth two and constraints for limiting the number of branching nodes. 
\citet{kiossou2022dl8.5_anytime} and \citet{demirovic2023anytime_blossom} improve the anytime performance of the search.
\citet{linden2023streed} generalize previous DP approaches beyond classification.
Because of all these algorithmic advances, a recent survey considers the DP approach currently the best in terms of scalability \citep{costa2023dt_survey}.
However, all of these methods require binarized input data.

\paragraph{Continuous features}
Quant-BnB \citep{mazumder2022odt_continuous_bnb} is the only specialized algorithm for ODTs that directly optimizes datasets with continuous features. It employs branch-and-bound by splitting on quantiles of the feature distribution. Although Quant-BnB can handle much larger datasets than the MIP and SAT approaches, it also struggles to scale beyond trees of depth three.

\paragraph{Summary} Scalability advances for ODT search mostly require binarization. When operating directly on the numeric data, scalability is still challenging.

%%%%%%%%%%%%%% Preliminaries %%%%%%%%%%%%%%%

\section{Preliminaries}   
In this section, we introduce notation, formally define the problem, 
and describe the lower-bounding technique that provides the basis for ConTree's pruning.

\paragraph{Notation}
Let $\mathcal{D}$ describe a dataset with $n = |\mathcal{D}|$ observations $(x, y)$ with $x \in \mathbb{R}^p$ and $y \in \mathcal{Y}$ describing respectively the feature vector and label of an observation. Here~$p$ is the number of features and $\mathcal{Y}$ the set of class labels. The set of all features is denoted as $\mathcal{F} = \{f_1, \ldots, f_p\}$. Each observation $x$ contains the values of these $p$~features such that $x_f$ is the value of feature $f$ in observation $x$. 
Let $\mathcal{D}^f$ denote the sorted values $x_{f}$ for $(x, y) \in \mathcal{D}$ and let $U^f$ describe all unique sorted values in $\mathcal{D}^f$ (with $\mathcal{D}$ determined by context) and similarly
\begin{equation}
S^f = \left\{ \frac{U^f_1 + U^f_2}{2}, ..., \frac{U^f_{|U^f|-1} + U^f_{|U^f|}}{2} \right\} 
\end{equation}%
the set of possible thresholds on feature $f$ for $\mathcal{D}$: the midpoints between the unique feature values. Let $m = |S^f|$ be the number of possible thresholds.
Given a threshold $\tau \in S^f$, let $z(\tau)$ represent the index of the observation in $\mathcal{D}^f$ with the largest value for $f$ smaller than $\tau$.
Similarly, let $u(\tau)$ represent the index of $\tau$ in $U^f$.
Finally, let $\mathcal{D}(f \leq \tau)$ describe the subset of observations $(x,y) \in \mathcal{D}$ where $x_f \leq \tau$.

\paragraph{Problem definition}
A decision tree is a function $t : \mathbb{R}^p \rightarrow \mathcal{Y}$ that recursively splits the feature space $\mathbb{R}^p$ into sub-regions and predicts the label of each sub-region.
Let $\mathcal{T}(\mathcal{D}, d)$ describe the set of all decision trees for the dataset $\mathcal{D}$ with a maximum depth of $d$. Then the ODT $t_{\operatorname{opt}}$ is the tree within that set that minimizes the misclassification score:
\begin{equation}
t_{\operatorname{opt}} = \argmin_{t \in \mathcal{T}(\mathcal{D}, d)} \sum_{(x,y) \in \mathcal{D}} \mathbb{1}(t(x) \neq y) \,.
\end{equation}%
This work is limited to binary axis-aligned trees: every branching node splits on precisely one feature $f \in \mathcal{F}$ based on a threshold $\tau$ such that every observation with $x_{f} \leq \tau$ goes left in the tree while the rest goes to the right.

\begin{figure*}
\centering
\begin{subfigure}{.49\textwidth}
  \centering
  \includegraphics[width=.98\linewidth]{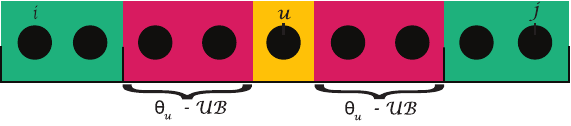}
  \caption{Neighborhood pruning}
  \label{fig:lower_bounds_neighborhood_pruning}
\end{subfigure}%
\begin{subfigure}{.49\textwidth}
  \centering
  \includegraphics[width=.98\linewidth]{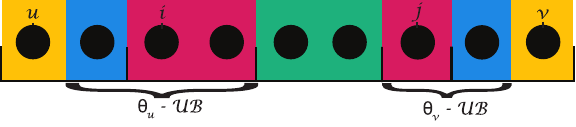}
  \caption{Interval shrinking}
  \label{fig:lower_bounds_interval_shrinking}
\end{subfigure}
\caption{The split points $u$ and $v$ for which the score $\theta_u$ and $\theta_v$ are calculated are yellow. The newly pruned values are shown in red. Green indicates the remaining split points for further search. Blue indicates unaffected values outside of $[i..j]$.}
\end{figure*}

\paragraph{Similarity lower bounding}
\citet{hu2019sparse}, \citet{lin2020generalized_sparse}, and \citet{demirovic2022murtree} propose a similarity-based lower-bounding (SLB) technique that compares a new dataset $\mathcal{D}_{new}$ with a previously analyzed dataset $\mathcal{D}_{old}$ to derive a lower bound on the misclassification score of the new dataset. SLB assumes that all new observations in the new dataset will be classified correctly and all removed observations from the old dataset were classified incorrectly, yielding the following lower bound: 
\begin{equation}
\label{eq:slb}
\theta_{\mathcal{D}_{new}} \geq \theta_{\mathcal{D}_{old}} - |\mathcal{D}_{old} \setminus \mathcal{D}_{new}| \,,
\end{equation}%
where $\theta$ is the misclassification score a decision tree with the same depth limit can achieve on the dataset. 
From this property, we derive three novel pruning techniques below.

%%%%%%%%% METHOD %%%%%%%%%%%%%
\section{The ConTree Algorithm}
We present the ConTree algorithm (CT) which constructs an ODT by recursively performing splits on every branching node within a full tree of pre-defined depth. Subproblems are identified by the dataset $\mathcal{D}$ and the remaining depth limit~$d$.
This results in the following recursive DP formulation: 
\begin{equation}
\label{eq:dp}
    \operatorname{CT}(\mathcal{D}, d) = 
    \begin{dcases}
        \min_{\hat{y} \in \mathcal{Y}} \sum_{(x,y) \in \mathcal{D}} \mathbb{1}(\hat{y} \neq y) & \text{if~} d=0 \\
        \min_{f \in \mathcal{F}} \operatorname{Branch}(\mathcal{D}, d, f)
         & \text{if~} d > 0
    \end{dcases}
\end{equation}%
Leaf nodes assign the label with the least misclassifications. Branching nodes find the feature~$f$ with the best misclassification score from the subtrees by calling the subprocedure $\operatorname{Branch}$ which iterates over all possible split thresholds~$\tau$: 
\begin{equation}
\label{eq:branch}
    \operatorname{Branch}(\mathcal{D}, d, f) = \min_{\tau \in S^f} \operatorname{Split}(\mathcal{D}, d, f, \tau) \,.
\end{equation}%
Every split results in two subproblems:
\begin{multline}
\label{eq:split}
    \operatorname{Split}(\mathcal{D}, d, f, \tau) = \\
    \operatorname{CT}\left(\mathcal{D}(f \leq \tau), d-1\right) + 
\operatorname{CT}\left(\mathcal{D}(f > \tau), d-1\right) \,.
\end{multline}%

Given a splitting feature $f$, computing the misclassification score $\theta_{\tau}$ for all possible split points $\tau \in S^f$ is computationally expensive since each split point considered requires solving two (potentially large) subproblems. Therefore, we provide the following runtime improvements (each of which preserves optimality):
\begin{description}
    \item[Lower-bound pruning] three novel pruning techniques specifically designed for continuous features to speed up the computation without losing optimality;
    \item[Depth-two subroutine] a subroutine for depth-two trees that iterates over sorted feature data to update class occurrences and efficiently solves the two depth-one subproblems in Eq.~\eqref{eq:split} simultaneously;
    \item[Caching] the same dataset caching technique as \citet{demirovic2022murtree}: ConTree reuses cached solutions to subproblems (defined by $\mathcal{D}$ and $d$).
\end{description}%
To control the trade-off between training time and accuracy, we also provide a $\operatorname{max-gap}$ parameter to set the maximum permissible gap to the optimal solution.

\subsection{Pruning Techniques}
\label{sec:pruning}
Based on the similarity-based lower bound presented in Eq.~\eqref{eq:slb}, we present three novel pruning techniques to reduce the number of split points that need to be considered in Eq.~\eqref{eq:branch} without losing optimality. The key idea is that if the feature data is sorted, the solution of any $\operatorname{Split}$ call with threshold $\tau$ provides a lower bound for all next calls with a different threshold $\tau'$ since we can easily count how many observations shifted from the left to the right subtree by subtracting their indices in the sorted data: $|z(\tau) - z(\tau')|$.
\begin{theorem}
\label{th:slb}
Let $\mathcal{UB}$ be the best solution so far or the score needed to obtain a better solution. Let $\theta_{\tau} = \operatorname{Split}(\mathcal{D}, d, f, \tau)$ be the optimal misclassification score for the subtree when branching on $f$ with threshold $\tau$. Then any other threshold $\tau'$ with $|z(\tau) - z(\tau')| \leq \theta_{\tau} - \mathcal{UB}$ cannot yield an improving solution.
\end{theorem}
\begin{proof} 
This follows directly from Eq.~\eqref{eq:slb}. If $\tau' > \tau$, then $|\mathcal{D}(f \leq \tau') \setminus \mathcal{D}(f \leq \tau)| = z(\tau') - z(\tau)$ and if $\tau' < \tau$, then $|\mathcal{D}(f > \tau') \setminus \mathcal{D}(f > \tau)| = z(\tau) - z(\tau')$. Therefore, the SLB for a split at $\tau'$ is: $
\theta_{\tau'}  \geq \theta_{\tau} - |z(\tau) - z(\tau')| \geq \mathcal{UB}$.
\end{proof}%
\begin{corollary}
\label{cor:nb}
Let $u$ be a split index with its corresponding solution value $\theta_u$ and index $z(u)$ within $\mathcal{D}^f$.
Let $\Delta$ be the difference between $\theta_u$ and $\mathcal{UB}$ and at least one: $\Delta = \max(1, \theta_u - \mathcal{UB})$. Any improving split must have a threshold smaller than the value in $\mathcal{D}^f$ at index $z(u) - \Delta$ or larger than the value at index $z(u) + \Delta$.
\end{corollary}%
\noindent
Per $\operatorname{Branch}$ call, ConTree keeps track of the set of threshold indices that may yield a split that improves the current best tree. This set is represented as a set of index intervals $Q$. 
Initially, $Q$ contains one interval of all indices: $Q = \{ [1 .. m] \}$. After each $\operatorname{Split}$ call, ConTree updates the set $Q$ by using three pruning functions~$\mathcal{P}$ that return a list of pruned intervals from the current interval~$[i .. j]$, the current~$\mathcal{UB}$, and one or more subproblem solutions $\theta$.
Next, we explain these three pruning functions: neighborhood pruning, interval shrinking, and sub-interval pruning.

\subsubsection{Neighborhood pruning (NB)}
\label{subsec:neighborhood_pruning}
After the misclassification score $\theta_{u}$ for a split point $u \in [i .. j]$ is computed, neighborhood pruning uses the SLB to remove similar split points from consideration. A simplified illustration of this pruning technique can be seen in Fig.~\ref{fig:lower_bounds_neighborhood_pruning}.
Using Cor.~\ref{cor:nb}, we define two functions $\ubar{A}$ and $\bar{A}$ that return the closest thresholds from $u$ that could still improve on $\mathcal{UB}$. 
\begin{align}
    \ubar{A}(u, \Delta) &= \max \left\{u' \in [m] ~|~ U^f_{u'} < \mathcal{D}^f_{z(u) - \Delta} \right\} \\
    \bar{A}(u, \Delta) &= \min \left\{u' \in [m] ~|~ U^f_{u'+1} > \mathcal{D}^f_{z(u) + \Delta} \right\}
\end{align}%
The functions $\ubar{A}$ and $\bar{A}$ can be implemented using binary search with time complexity $\mathcal{O}(\log(m))$. Using these, we define $\mathcal{P}_{\operatorname{NB}}$ which yields two new intervals:
\begin{equation}
    \mathcal{P}_{\operatorname{NB}}([i .. j], u, \Delta) = \{ [i .. \ubar{A}(u, \Delta)], [\bar{A}(u, \Delta) .. j] \}
\end{equation}%
\begin{example}
Consider a continuous feature vector with values $[0.4$, $0.5$, $0.5$, $0.7$, $0.8$, $0.10]$ with split points $[0.45$, $0.6$, $0.75$, $0.9]$, where we have computed an optimal tree for the split on $\tau = 0.75$ with two more misclassifications than the current best solution $\mathcal{UB}$, that is, $\Delta = 2$. Therefore, any possibly improving split needs to move at least two instances from the left to the right (or vice versa), which means that only the split point $\tau = 0.45$ can yield an improving solution.
\end{example}

\subsubsection{Interval shrinking (IS)}
\label{subsec:interval_shrinking}
Interval shrinking is an extension of neighborhood pruning and acts whenever $\mathcal{UB}$ is updated.
Given an interval $[i .. j] \in Q$, IS searches for the largest threshold index $u \in \mathcal{V}$ smaller than $i$ and the smallest threshold index $v \in \mathcal{V}$ larger than $j$, with $\mathcal{V}$ the set of split indices for which the misclassification score $\theta_u$ and $\theta_v$ are already computed. Using Cor.~\ref{cor:nb}, IS then prunes the interval $[i .. j]$ as illustrated in Fig.~\ref{fig:lower_bounds_interval_shrinking}.
To search for these indices $u$ and $v$, we define the function $B$ that uses binary search over $\mathcal{V}$ with time complexity $\mathcal{O}(\log(m))$:
\begin{equation}
\label{eq:outside_indices}
    B([i .. j], \mathcal{V}) = \Big\langle \max_{ u \in \mathcal{V} : u < i } u, \min_{ v \in \mathcal{V} : v > j } v \Big\rangle \,.
\end{equation}%
Additionally, IS uses the following theorem:
\begin{theorem}
\label{th:zero}
Let $w$ by any split point with a solution $\theta_w$ with a left subtree misclassification score $\theta_{w, L}$ of zero. Then any split point $u < w$ will yield $\theta_u \geq \theta_w$. Similarly, if the right subtree misclassification score $\theta_{w, R}$ is zero, then for all $v > w$ also $\theta_v \geq \theta_w$.
\end{theorem}
\begin{proof}
W.l.o.g., consider the left case. Since $u < w$, it holds that $\mathcal{D}(f \leq S^f_u) \subset \mathcal{D}(f \leq S^f_w)$ and $\mathcal{D}(f > S^f_u) \supset \mathcal{D}(f > S^f_w)$.
Eq.~\eqref{eq:slb} then implies that $\theta_{u, L} \leq \theta_{w, L} = 0$ and $\theta_{u, R} \geq \theta_{w, R}$. Thus $\theta_u = \theta_{u, R}$, $\theta_w = \theta_{w, R}$, and $\theta_u \geq \theta_w$.
\end{proof}%
Let $M_L$ and $M_R$ represent the right-most (left-most) index with a zero left (right) misclassification score found so far, incremented (decremented) by one. By combining Cor.~\ref{cor:nb} and Theorem~\ref{th:zero}, we define:
\begin{multline}
    \mathcal{P}_{\operatorname{IS}}([i .. j], u, v, \Delta_u, \Delta_v, M_L, M_R) = \\ [ \max(i, M_L, \bar{A}(u, \Delta_u)), \min(j, M_R, \ubar{A}(v, \Delta_v))] \,,
\end{multline}
where $\Delta_u = \theta_u - \mathcal{UB}$ and $\Delta_v = \theta_v - \mathcal{UB}$.

\subsubsection{Sub-interval pruning (SP)}
\label{subsec:sub_interval_pruning}
Sub-interval pruning can prune an entire interval $[i .. j]$ using  the following theorem:
\begin{theorem}
\label{th:sp}
Let $[i .. j]$ be any threshold index interval. Let $\theta_u$ and $\theta_v$ be optimal solutions for previously computed split points $u < i$ and $v > j$, with the corresponding left and right misclassification scores $\theta_{u, L}$, $\theta_{u, R}$, $\theta_{v, L}$, and $\theta_{v, R}$. Then, if $\theta_{u, L} + \theta_{v,R} > \mathcal{UB}$, any split point $w \in [i .. j]$ cannot improve on $\mathcal{UB}$.
\end{theorem}%
\begin{proof}
$w \geq i > u$ and $w \leq j < v$ and thus $\theta_{w, L} \geq \theta_{u, L}$ and $\theta_{w, R} \geq \theta_{v, R}$. Therefore, $\theta_w = \theta_{w,L} + \theta_{w,R} > \mathcal{UB}$.
\end{proof}%
\noindent
From Theorem~\ref{th:sp}, we define:
\begin{multline}
\label{eq:sp_pruning}
\mathcal{P}_{SP}([i .. j], \mathcal{UB}, \theta_{u, L}, \theta_{v, R}) = \\
\begin{cases}
\emptyset, & \text{if } \theta_{u, L} + \theta_{v, R} > \mathcal{UB} \land u < i \land v > j \\
[i .. j], & \text{otherwise} \,.
\end{cases}
\end{multline}

\begin{algorithm}[t!]
\caption{$\operatorname{Branch}(\mathcal{D}, d, f, \mathcal{UB})$ uses the pruning techniques to find the optimal threshold for feature~$f$.}
\label{alg:branch_general}
\DontPrintSemicolon
    $\theta_{\operatorname{opt}} \leftarrow \min_{\hat{y} \in \mathcal{Y}} \sum_{(x,y) \in \mathcal{D}} \mathbb{1}(\hat{y} \neq y)$ \;
    $M_L \leftarrow 0, M_R \leftarrow m+1$\;
    $Q \leftarrow \{ [1 .. m] \}, \mathcal{V} \leftarrow \emptyset $\;
    \While{$|Q| > 0$} {
        $[i .. j] \leftarrow Q.\operatorname{pop}()$\;
        $u, v \leftarrow B([i .. j], \mathcal{V})$ \;
        $\Delta_u \leftarrow \theta_u - \mathcal{UB}, \Delta_v \leftarrow \theta_v - \mathcal{UB}$\;
        $[i .. j] \leftarrow \mathcal{P}_{\operatorname{IS}}([i..j], u, v, \Delta_u, \Delta_v, M_L, M_R)$\;
        $[i .. j] \leftarrow \mathcal{P}_{\operatorname{SP}}([i .. j], \mathcal{UB}, \theta_{u, L}, \theta_{v,R})$\;
        \lIf{$|[i..j]| = 0$}{\Continue}
        $w \leftarrow \lfloor \frac{i+j}{2} \rfloor$\;
        \lIf{$d=2$} { $\theta_{w, L}, \theta_{w, R} \leftarrow \operatorname{D2Split}(\mathcal{D}, f, w)$  }
        \uElse{
            $\mathcal{D}_L \leftarrow \mathcal{D}(f \leq S^f_w), \mathcal{D}_R \leftarrow \mathcal{D}(f > S^f_w)$\;
            $\theta_{w,L} \leftarrow \operatorname{CT}(\mathcal{D}_L, d-1, \mathcal{UB})$\;
            $\eta \leftarrow \min(z(w) - z(i), z(j) - z(w))$\;
            $\mathcal{UB}_R \leftarrow \max(\mathcal{UB} - \theta_{w,L}, \eta)$\;
            \lIf{$\mathcal{UB}_R \leq 0$} { $\theta_{w, R} \leftarrow \theta_{\operatorname{opt}} - \theta_{w, L}$}
            \lElse{
            $\theta_{w,R} \leftarrow \operatorname{CT}(\mathcal{D}_R, d-1, \mathcal{UB}_R)$}
        }
        $\theta_w \leftarrow \theta_{w,L} + \theta_{w,R}$\;
        \lIf{$\theta_{w, L} = 0$}{$M_L \leftarrow w+1$}
        \lIf{$\theta_{w, R} = 0 \land \mathcal{UB}_R > 0$}{$M_R \leftarrow w-1$}
        \uIf{$\theta_w < \theta_{\operatorname{opt}}$} {
            $\mathcal{UB} \leftarrow \min(\mathcal{UB}, \theta_w), \theta_{\operatorname{opt}} \leftarrow \theta_w$\;
            \lIf{$\theta_{\operatorname{opt}} = 0$}{\Break}
        }
        $Q.\operatorname{push}\left(\mathcal{P}_{\operatorname{NB}}([i..j], w, \max(1, \theta_w - \mathcal{UB}))\right)$\;  
        $\mathcal{V} \leftarrow \mathcal{V} \cup \{ w \}$ \;
    }
    \Return $\theta_{\operatorname{opt}}$
\end{algorithm}%
\subsection{General Recursive Case}
\label{subsec:general_solver}
To eliminate the exploration of unnecessary splits in the recursion of Eq.~\eqref{eq:dp}, we use the pruning mechanisms presented above and keep track of upper bounds.

\paragraph{Interval pruning} 
In each $\operatorname{Branch}$ call, described in Eq.~\eqref{eq:branch}, for each feature, we keep track of a set of intervals~$Q$ that may still contain the optimal split. We then choose a split point from one of these intervals, compute an optimal solution for this split, and then prune the search space according to the techniques mentioned above. When $Q$ is empty, we have arrived at the optimal solution.

The pseudo-code of this procedure is shown in Alg.~\ref{alg:branch_general}. It loops over the set of intervals $[i .. j] \in Q$. 
Using Eq.~\eqref{eq:outside_indices}, it finds $u, v \in \mathcal{V}$, the closest indices outside of $[i .. j]$ for which the misclassification score is already computed.
First, it attempts to reduce the interval using SP and IS. Then we need to select any point $w \in [i .. j]$. If $w$ is close to previously computed splits, only a little new information is gained. Therefore, we choose the midpoint $w = \lfloor\frac{i+j}{2}\rfloor$. 

Then, if the remaining tree depth budget is two, we use a special subroutine which is explained in the subsection below. Otherwise, the dataset is split and a recursive $\operatorname{CT}$ call is made to get the optimal left subtree. This left solution is used to compute an upper bound for the right subtree. Only if the right upper bound indicates that an improving solution can still be found, is another recursive call made to get the right optimal subtree. If either the left or right subtree has zero misclassifications, the $M_L$ or $M_R$ indices are updated accordingly. 
If a better solution is found, the upper bound is updated. If a tree with zero misclassifications is found, the search is done.
Finally, the remaining interval is divided into two using NB, both of which are added to $Q$, and the search continues with the next interval in $Q$.

\begin{algorithm}[t!]
\caption{$\operatorname{D2Split}(\mathcal{D}, f_1, w)$ finds splits of depth two more efficiently than the normal recursion.}
\label{alg:specialized_solver}
\DontPrintSemicolon
    $\theta_L \leftarrow |\mathcal{D}|, \theta_R \leftarrow |\mathcal{D}|$\;
    $FQ^y_L \leftarrow 0 \quad \forall y \in \mathcal{Y}$\;
    \lFor{$(x, y) \in \mathcal{D}(f \leq S^{f_1}_w)$}{
    $FQ^y_L \leftarrow FQ^y_L + 1$
    }
    $FQ^y_R \leftarrow FQ^y - FQ^y_L \quad \forall y \in \mathcal{Y}$\;
    \For{$f_2 \in \mathcal{F}$} {
    $C^y_{L} \leftarrow 0, C^y_{R} \leftarrow 0 \quad \forall y \in \mathcal{Y}$\;
    \For{$(x,y) \in \mathcal{D} \text{~sorted by~}f_2$} {
        \uIf{$x_{f_1} \leq S^{f_1}_w$} {
            $\theta_{LL} \leftarrow \min_{\hat{y}}(C^{\hat{y}}_{L})$\;
            $\theta_{LR} \leftarrow  \min_{\hat{y}}(FQ^{\hat{y}}_L - C^{\hat{y}}_{L})$\;
            \uIf{$\theta_{LL} + \theta_{LR} \leq \theta_L$} {
                $\theta_L \leftarrow \theta_{LL} + \theta_{LR}$\;
            }
            $C^y_L \leftarrow C^y_L + 1$\;
        }
        \uElse{
            Same logic for instances going right
        }
        \lIf{$\theta_L + \theta_R = 0$}{\Break}
    }
    }
    \Return $\theta_L, \theta_R$

\end{algorithm}%
\paragraph{Upper bounds}
If we set the upper bound for the right subtree tightly to the bound of what right solution can improve the current solution ($\mathcal{UB}_R \leftarrow \mathcal{UB} - \theta_{w, L}$), the right subproblem can be terminated early if no such solution exists. However, when doing so, we do not gain any information for pruning and we observed that this therefore typically decreases performance.
On the other hand, no upper-bound-based pruning at all results in many unnecessary recursive calls.
In this trade-off, we settled on a hybrid approach that in practice works well: the right upper bound $\mathcal{UB}_R$ is set to the maximum of $\mathcal{UB} - \theta_{w, L}$ and the length $\eta$ of the longest side of the interval edges $z(i)$ and $z(j)$ to the midpoint $z(w)$.

\subsection{Depth-Two Subroutine}
\label{subsec:d2_specialized_solver}
To improve runtime, \citet{demirovic2022murtree} introduced a specialized subroutine for trees of depth-two that is more efficient than doing recursive calls. However, it requires a quadratic amount of memory in terms of the number of binary features, which is problematic if we consider a binary feature for every possible threshold on continuous data.

Instead, we provide a specialized subroutine that simultaneously finds an optimal left and right subtree of a depth-two split and does not have this quadratic memory consumption by exploiting the fact that we can sort the observations by their feature values. 
Since splitting the data preserves the order, we sort the dataset once in the beginning.
Then for the sorted data, Alg.~\ref{alg:specialized_solver} shows how an optimal depth-two tree can be found in $\mathcal{O}(|\mathcal{D}||\mathcal{F}|)$ for a given split point~$w$ on feature~$f_1$ for the root node of the subtree.
The core idea is that we first count with the variables $FQ^y_L$ and $FQ^y_R$ how many observations of class $y$ go to the left and right by splitting at point~$w$.
Then, when deciding on the second splitting feature $f_2$ (of either the left or right subtree), we traverse all observations sorted by $f_2$ and incrementally keep track of the current counts per label $C^y_L$ and $C^y_R$.
Based on these two label counts, the label counts for all splitting thresholds on $f_2$ for all four leaf nodes of a depth-two tree can be determined as the dataset is traversed:
\begin{equation}
\begin{aligned}
C^y_{LL} &= C^y_L \,, \quad & \quad
C^y_{LR} &= FQ^y_L - C^y_L \,, \\
C^y_{RL} &= C^y_R \,, \quad & \quad
C^y_{RR} &= FQ^y_R - C^y_R \,.
\end{aligned}%
\end{equation}%
Alg.~\ref{alg:specialized_solver} shows how the minimal misclassifications in the subtrees $\theta_{LL}$ and $\theta_{RR}$ can be computed based on these values.

\subsection{Optimality Gap} We add a $\operatorname{max-gap}$ parameter that determines how far from optimal the final solution is allowed to be. This increases the pruning strength at the expense of optimality. For example, for NB, the distance $\Delta$ to a possibly improving split is now computed as $\Delta = \max(1, \theta_u - (\mathcal{UB} - \operatorname{max-gap}))$. 
To use the gap parameter across multiple depths of the search, we set the $\operatorname{max-gap}$ for the current depth to half of the total. The other half is distributed evenly over the two subproblems.
This allows a trade-off between training time and accuracy.

\subsection{Comparison to Previous Bounding Methods}
\citet{mazumder2022odt_continuous_bnb} propose three lower bounds that require splitting the data into quantiles for a given feature. Their first lower bound is similar to our sub-interval pruning, but our definition is independent of the remaining depth budget. Their other lower bounds are tighter, but also more expensive to compute. Future work could investigate using such or similar lower bounds in ConTree.

The similarity-based lower bound (SLB) was proposed in previous work \citep{hu2019sparse, lin2020generalized_sparse, demirovic2022murtree}, but our application of the bound is more efficient. 
\citet{lin2020generalized_sparse} point out that the \emph{similar support} bound in OSDT \citep{hu2019sparse} is too expensive to compute frequently. Therefore, they propose to use \emph{hash trees} to identify similar subtrees but provide no further details. The implementation of their method, GOSDT, computes the difference between two subproblems by computing the xor of bit-vectors that represent the dataset corresponding to the subproblems. \citet{demirovic2022murtree} loop once over the two sorted lists of identifiers of the datasets to count the differences. Both of these approaches require $\mathcal{O}(n)$ operations to compute the difference. ConTree, on the other hand, exploits the properties of sorted numeric feature data and computes the bound in $\mathcal{O}(1)$ by computing the difference between the two split indices. It applies the bound using the novel pruning techniques in $\mathcal{O}(\log(m))$.

%%%%%%%%%%%%%%%%% EXPERIMENTS %%%%%%%%%%%%%%%%%%%%%%%%
\begin{table*}[t!]
\centering
\begin{tabular}{l ccc cccccc}
\toprule

Dataset &
$|\mathcal{D}|$ & $|\mathcal{F}|$ & $|\mathcal{Y}|$ &
OCT & RS-OCT & SAT-Shati & Quant-BnB & ConTree (No D2) & ConTree \\ \midrule
                         
Avila          & 10430  & 10 & 12 & 
(100\%) & (46\%) & > 4h & 2 & 63 & \textbf{0.1} \\
Bank           & 1097   & 4  & 2  & 
(31\%) & 177 & 62 & 0.5 & 0.1 & \textbf{< 0.1} \\
Bean           & 10888  & 16 & 7  & 
(100\%) & (24\%) & > 4h & 4 & 199 & \textbf{0.2}       \\
Bidding        & 5056   & 9  & 2  & 
(100\%) & (67\%) & 672 & 0.9 & 1 & \textbf{0.1}  \\
Eeg            & 11984  & 14 & 2  & 
OoM & (165\%)  & > 4h & 4 & 178 & \textbf{0.2}          \\
Fault          & 1552   & 27 & 7  & 
(100\%) & (126\%)  & > 4h & 2 & 240 & \textbf{0.1}        \\
Htru           & 14318  & 8  & 2  & 
OoM & (301\%)  & > 4h & 2 & 929 & \textbf{0.2}           \\
Magic          & 15216  & 10 & 2  & 
(100\%) & (88\%)  & > 4h & 2 & 42 & \textbf{0.3}        \\
Occupancy      & 8143   & 5  & 2  & 
(100\%) & (9\%)  & 355 & 0.9 & 2 & \textbf{< 0.1}       \\
Page           & 4378   & 10 & 5  & 
(100\%) & (136\%) & 2836 & 1 & 3 & \textbf{< 0.1}        \\
Raisin         & 720    & 7  & 2  &
(99\%) & (15\%) & 485 & 0.7 & 1 & \textbf{< 0.1}        \\
Rice           & 3048   & 7  & 2  & 
(100\%) & (54\%)  & > 4h & 1 & 24 & \textbf{0.1}         \\
Room           & 8103   & 16 & 4  & 
(100\%) & (88\%) & 4327 & 2 & 10 & \textbf{< 0.1}           \\
Segment        & 1848   & 18 & 7  & 
(100\%) & 2896  & 442 & 2 & 54 & \textbf{0.1}           \\
Skin           & 196045 & 3  & 3  & 
- & (43\%)  & > 4h & 3 & 63 & \textbf{0.1}            \\
Wilt           & 4339   & 5  & 5  & 
(100\%) & (35\%)  & 68 & 0.9 & 2 & \textbf{< 0.1}          \\
\bottomrule
\end{tabular}
\caption{Runtime (s) for optimizing depth-two trees of  OCT, RS-OCT, SAT-Shati, Quant-BnB, and ConTree (with and without the depth-two subroutine). Runtimes are averaged over twenty runs. Values in parentheses show the optimality gap at time out. OoM means out of memory. `-' means the linear relaxation was unsolved at the time limit.
Best results are in bold.}
\label{tab:runtime_comparison}
\end{table*}

\section{Experiments}
\label{sec:numerical_experiments}
Our experiments aim to answer the following questions:
1)~what is the effect of the pruning techniques and the depth-two subroutine;
2)~how does ConTree's runtime compare to state-of-the-art ODT algorithms; and
3)~what is ConTree's anytime performance?
In the appendices~B, C, and~D, we additionally answer the following questions: 
4)~how does ConTree's memory usage compare to previous methods;
5)~how well does ConTree scale to larger depth limits; and
6)~how does ConTree's out-of-sample performance compare to CART and ODTs on binarized data?

The results show that our pruning techniques and depth-two subroutine make ConTree one or more orders of magnitude faster than the state-of-the-art optimal methods while using significantly less memory. Additionally, its out-of-sample performance is better than both CART and ODTs on binarized data. Finally, good solutions are often found early but much time is spent proving optimality.

\paragraph{Setup} 
We have implemented ConTree in C++ and provide it as a Python package.\footnote{\url{https://github.com/ConSol-Lab/contree}.}
All computations were performed on an Intel Xeon E5-6448Y 32C 2.1GHz processor with 25GB RAM running Linux Red Hat Enterprise 8.10. We set a timeout of four hours. 
We evaluate on 16 datasets from the UCI repository \citep[][see Appendix~A]{uci2017}. Unless specified otherwise, we run ConTree with its $\operatorname{max-gap}$ set to zero, thus yielding optimal solutions.

\begin{figure}[t!]
    \centering
    \includegraphics[width=\columnwidth]{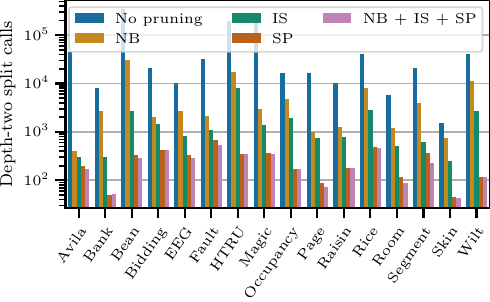}
    \caption{The number of $\operatorname{D2Split}$ calls for no pruning, the three pruning techniques, and all three combined.}
    \label{fig:lower_bound_comparison}
\end{figure}

\paragraph{Pruning techniques}
Fig.~\ref{fig:lower_bound_comparison} shows the impact of the neighborhood pruning (NB), interval shrinking (IS), and sub-interval pruning (SP) techniques in comparison to the baseline with no pruning for trees of depth two.
On average, all pruning techniques significantly prune the search space, with NB, IS, and SP respectively pruning on average 91.1\%, 97.5\%, and 99.6\% of all $\operatorname{D2Split}$ calls, while retaining the optimal solution.
SP generally provides the best results since it efficiently prunes entire intervals.
Combining all three methods (NB+IS+SP) on average yields a 10\% reduction of $\operatorname{D2Split}$ calls in comparison to only using SP.

\paragraph{Depth-two subroutine}
The last two columns of Table~\ref{tab:runtime_comparison} show the performance of ConTree with the depth-two subroutine $\operatorname{D2Split}$ versus without (ConTree with ``No D2'') for computing depth-two trees. Both methods use all the pruning techniques. 
Averaged over twenty runs, the depth-two solver improves the computation time by a factor of 320 compared to the baseline (geometric mean). 

\paragraph{Runtime}
\citet{mazumder2022odt_continuous_bnb} compared Quant-BnB with BinOCT \citep{verwer2019learning}, DL8.5 \citep{aglin2020learning}, and MurTree \citep{demirovic2022murtree}, each of which requires explicit binarization.  Quant-BnB scales one or more orders of magnitude better than all of these methods when those are trained with binary features for every possible threshold.
Additionally, they report that the optimal methods OCT \citep{bertsimas2017optimal}, GOSDT \citep{lin2020generalized_sparse}, FlowOCT, and BendersOCT \citep{aghaei2024strong} cannot solve any of their datasets within a four-hour time limit.

We compare ConTree to the MIP methods OCT and RS\nobreakdash-OCT \citep{hua2022odt_mip_branching}, the SAT method by \citet{shati2023sat_odt_nonbin}, and Quant-BnB \citep{mazumder2022odt_continuous_bnb}.\footnote{\url{https://github.com/LucasBoTang/Optimal_Classification_Trees}, \url{https://github.com/YankaiGroup/optimal_decision_tree}, \url{https://github.com/HarisRasul12/ESC499-Thesis-SAT-Trees}, \url{https://github.com/mengxianglgal/Quant-BnB}.} 
Each of these methods optimizes ODTs directly on the numeric feature data. 
We initialize both OCT and RS\nobreakdash-OCT with a warm start from CART and set the node cost to zero. Both OCT and RS\nobreakdash-OCT could use up to eight threads. OCT is solved with Gurobi 9.5.2.
Each method is run twenty times on each dataset, except when it exceeds the four-hour time-out. For the MIP methods, we report the optimality gap at time-out. 

Table~\ref{tab:runtime_comparison} shows the runtime results for optimizing trees of depth two. Even after four hours, the MIP methods typically have a large optimality gap remaining. OCT ran out of memory twice and once did not solve the linear relaxation before the time out.
The SAT approach performs better but also hits the time-out for seven datasets.
Quant-BnB and ConTree, on the other hand, run in the (sub-)second range. Therefore, we further compare Quant-BnB and ConTree.

Table~\ref{tab:runtime_depth34} shows that for depth three, ConTree outperforms Quant-BnB on average by a factor 63 (geometric mean), ranging from 7 times faster for Bidding up to 266 times faster for Fault. 
The comparisons of Quant-BnB with BinOCT, DL8.5, MurTree, GOSDT, OCT, FlowOCT, and BendersOCT by \citet{mazumder2022odt_continuous_bnb} combined with our new runtime results show that ConTree outperforms state-of-the-art optimal methods by one or more orders of magnitude.
\begin{table}[t!]
\centering
\setlength{\tabcolsep}{1mm}
\begin{tabular}{l cccc}
\toprule

&
\multicolumn{2}{c}{Depth = 3} &
\multicolumn{2}{c}{Depth = 4} \\

\cmidrule(lr){2-3}
\cmidrule(lr){4-5}

Dataset & Quant-BnB & ConTree & ConTree & $\leq 1\%$ Gap \\ \midrule
                         
Avila          &
4451 & 24    & 4195 & 3237  \\
Bank           &
2    & < 0.1 & < 0.1 & < 0.1\\
Bean           &
583  & 61    & > 4h & 3640  \\
Bidding        &
15   & 2     & 5 & < 0.1    \\
Eeg            &
8535 & 136   & > 4h &  > 4h  \\
Fault          &
> 4h & 55    & 12331 & 8592  \\
Htru           &
13147 & 74   & > 4h  & 191 \\
Magic          &
1419 & 60   & > 4h & 5719  \\
Occupancy      &
76  & 0.4   & 17 & 0.1  \\
Page           &
388  & 2     & 499   & 63 \\
Raisin         &
65  & 0.5   & 65    & 27 \\
Rice           &
817 & 12    & 2215  & 231 \\
Room           &
92  & 1     & 44   & 0.1  \\
Segment        &
64  & 2     & 191   & 75 \\
Skin           &
218  & 10    & 211  & 5 \\
Wilt           &
26   & 0.1   & 0.3    & < 0.1             \\
\bottomrule
\end{tabular}
\caption{Runtime (s) comparison between Quant-BnB and ConTree. Averaged over twenty runs.
For depth three, ConTree is on average one or two orders of magnitude faster than Quant-BnB. ConTree's runtime can be significantly reduced by setting a permissible optimality gap.}
\label{tab:runtime_depth34}
\end{table}
\begin{figure}[t!]
    \centering
    \includegraphics[width=\linewidth]{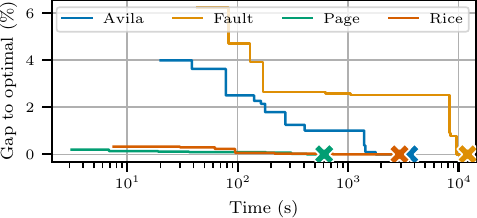}
    \caption{The distance to the optimal solution for ConTree's best depth-four solution over time for three datasets. The optimal solution is typically found significantly earlier than the end of the search (the final cross). }
    \label{fig:anytime}
\end{figure}

Quant-BnB's implementation only permits training trees up to depth three, so for depth four, we report only ConTree's performance. For all but four datasets, ConTree finds the optimal depth-four tree within the time limit.
Appendix~B further shows that for depth five and six, ConTree finds and proves optimal solutions within the time limit for eight and seven datasets respectively.

The last column in Table~\ref{tab:runtime_depth34} shows the runtime ConTree requires to find a depth-four tree that is provably within 1\% of the optimal solution (with $\operatorname{max-gap} = \lfloor 0.01 |\mathcal{D}| \rfloor$). This significantly reduces ConTree's runtime and allows a user to make a trade-off between training accuracy and runtime.

\paragraph{Anytime performance} Fig.~\ref{fig:anytime} shows the best solution found by ConTree at any time during a depth-four search, expressed as the distance to the optimal solution for the Avila, Fault, Page, and Rice datasets.
For the Fault dataset, a final improving solution is found late after the start of the search.
In the other three cases, a good solution is found early and most of the time is spent on proving optimality. 

%% Cite datasets
\nocite{misc_avila_459, misc_banknote_authentication_267, bean_dataset, misc_shill_bidding_dataset_562, misc_eeg_eye_state_264, misc_steel_plates_faults_198, htru_dataset, misc_magic_gamma_telescope_159, occupancy_dataset, misc_page_blocks_classification_78, raisin_dataset, rice_dataset, room_dataset, misc_image_segmentation_50, misc_skin_segmentation_229, misc_wilt_285}

%%%%%%%%%%%%%% CONCLUSION %%%%%%%%%%%
\section{Conclusion }
\label{sec:conclusion_and_discussions}
We introduce ConTree, a dynamic programming and branch-and-bound algorithm that outperforms state-of-the-art algorithms for training optimal classification trees with continuous features by one or more orders of magnitude. 
This result is obtained through three novel pruning techniques that on average prune 99.6\% of recursive calls, and a depth-two subroutine that computes splits 320 times faster than a naive approach.
Moreover, ConTree obtains an out-of-sample accuracy for depth-three trees 5\% higher than CART and 0.7\% higher than optimal decision trees trained with a coarse binarization. 
These results enable the application of optimal decision trees for real-world scenarios.

Future work could add a node cost or node limit to further encourage sparsity and extend ConTree to regression by using bounds from \citet{zhang2022sparse_regtrees} and \citet{bos2024srt}.
Finally, branching on the power set of categorical variables as done by \citet{shati2023sat_odt_nonbin}, could be explored.

\bibliography{references}

\newpage
\appendix
\input{appendix.tex}

\end{document}

%% file: appendix.tex
\begin{table*}[t!]
    \centering
    \begin{tabular}{l ccc ccc ccc ccc}
    \toprule
    &
    \multicolumn{3}{c}{Depth = 5, Optimal} &
    \multicolumn{3}{c}{Depth = 5, $\leq 1\%$ Gap} \\ 
    \cmidrule(lr){2-4}
    \cmidrule(lr){5-7}

    Dataset &
    Accuracy (\%) & Runtime (s) & Memory (MB) &
    Accuracy (\%) & Runtime (s) & Memory (MB) \\ \midrule

        Avila &
    \textbf{67.9} & % optimal accuracy
    > 4h & % optimal runtime (s)
    44 & % optimal memory (MB)
    \textbf{67.9} & % 1% gap accuracy
    > 4h & % 1% gap runtime (s)
    41 \\ % 1% gap memory (MB)

    Bank &
    \textbf{100} & % optimal accuracy
    < 0.1 & % optimal runtime (s)
    4 & % optimal memory (MB)
    99.9 & % 1% gap accuracy
    < 0.1 & % 1% gap runtime (s)
    4 \\ % 1% gap memory (MB)

    Bean &
    \textbf{92.4} & % optimal accuracy
    > 4h & % optimal runtime (s)
    34 & % optimal memory (MB)
    \textbf{92.4} & % 1% gap accuracy
    > 4h & % 1% gap runtime (s)
    33 \\ % 1% gap memory (MB)

    Bidding &
    \textbf{100} & % optimal accuracy
    10 & % optimal runtime (s)
    11 & % optimal memory (MB)
    99.9 & % 1% gap accuracy
    < 0.1 & % 1% gap runtime (s)
    9 \\ % 1% gap memory (MB)

    Eeg &
    \textbf{75.5} & % optimal accuracy
    > 4h & % optimal runtime (s)
    30 & % optimal memory (MB)
    \textbf{75.5} & % 1% gap accuracy
    > 4h & % 1% gap runtime (s)
    29 \\ % 1% gap memory (MB)

    Fault &
    \textbf{76.4} & % optimal accuracy
    > 4h & % optimal runtime (s)
    12 & % optimal memory (MB)
    74.0 & % 1% gap accuracy
    > 4h & % 1% gap runtime (s)
    12 \\ % 1% gap memory (MB)

    Htru &
    98.4 & % optimal accuracy
    > 4h & % optimal runtime (s)
    37 & % optimal memory (MB)
    \textbf{98.5} & % 1% gap accuracy
    > 4h & % 1% gap runtime (s)
    61 \\ % 1% gap memory (MB)

    Magic &
    81.3 & % optimal accuracy
    > 4h & % optimal runtime (s)
    34 & % optimal memory (MB)
    \textbf{86.6} & % 1% gap accuracy
    > 4h & % 1% gap runtime (s)
    40 \\ % 1% gap memory (MB)

    Occupancy &
    \textbf{99.9} & % optimal accuracy
    119 & % optimal runtime (s)
    39 & % optimal memory (MB)
    99.7 & % 1% gap accuracy
    0.3 & % 1% gap runtime (s)
    10 \\ % 1% gap memory (MB)

    Page &
    \textbf{98.6} & % optimal accuracy
    > 4h & % optimal runtime (s)
    161 & % optimal memory (MB)
    \textbf{98.6} & % 1% gap accuracy
    3647 & % 1% gap runtime (s)
    55 \\ % 1% gap memory (MB)

    Raisin &
    \textbf{95.4} & % optimal accuracy
    1199 & % optimal runtime (s)
    50 & % optimal memory (MB)
    95.3 & % 1% gap accuracy
    489 & % 1% gap runtime (s)
    23 \\ % 1% gap memory (MB)

    Rice &
    \textbf{95.6} & % optimal accuracy
    > 4h & % optimal runtime (s)
    74 & % optimal memory (MB)
    95.5 & % 1% gap accuracy
    > 4h & % 1% gap runtime (s)
    84 \\ % 1% gap memory (MB)

    Room &
    \textbf{100} & % optimal accuracy
    7 & % optimal runtime (s)
    21 & % optimal memory (MB)
    99.9 & % 1% gap accuracy
    0.2 & % 1% gap runtime (s)
    18 \\ % 1% gap memory (MB)

    Segment &
    \textbf{99.1} & % optimal accuracy
    3715 & % optimal runtime (s)
    51 & % optimal memory (MB)
    99.0 & % 1% gap accuracy
    776 & % 1% gap runtime (s)
    20 \\ % 1% gap memory (MB)

    Skin &
    \textbf{99.8} & % optimal accuracy
    1705 & % optimal runtime (s)
    782 & % optimal memory (MB)
    99.6 & % 1% gap accuracy
    13 & % 1% gap runtime (s)
    139 \\ % 1% gap memory (MB)

    Wilt &
    \textbf{100} & % optimal accuracy
    < 0.1 & % optimal runtime (s)
    7 & % optimal memory (MB)
    99.9 & % 1% gap accuracy
    < 0.1 & % 1% gap runtime (s)
    7 \\ % 1% gap memory (MB)

    \bottomrule
    \end{tabular}
    \caption{ConTree results for maximum depth five. Runtime is averaged over five runs. The maximum memory usage observed is stated. If ConTree exceeds the time limit we report the best solution found at time-out. Best accuracy results are put in bold.
    }
    \label{tab:contree_d5}
\end{table*}

\section{Data Preprocessing}
\label{app:data}

\begin{table*}[t!]
    \centering
    \begin{tabular}{l ccc ccc}
    \toprule
    &
    \multicolumn{3}{c}{Depth = 6, Optimal} &
    \multicolumn{3}{c}{Depth = 6, $\leq 1\%$ Gap} \\ 
    \cmidrule(lr){2-4}
    \cmidrule(lr){5-7}

    Dataset &
    Accuracy (\%) & Runtime (s) & Memory (MB) &
    Accuracy (\%) & Runtime (s) & Memory (MB)  \\ \midrule
    
    Avila &
    \textbf{70.5} & % optimal accuracy
    > 4h & % optimal runtime (s)
    52 & % optimal memory (MB)
    \textbf{70.5} & % 1% gap accuracy
    > 4h & % 1% gap runtime (s)
    51 \\ % 1% gap memory (MB)

    Bank &
    \textbf{100} & % optimal accuracy
    < 0.1 & % optimal runtime (s)
    5 & % optimal memory (MB)
    \textbf{100} & % 1% gap accuracy
    < 0.1 & % 1% gap runtime (s)
    5 \\ % 1% gap memory (MB)

    Bean &
    \textbf{73.7} & % optimal accuracy
    > 4h & % optimal runtime (s)
    56 & % optimal memory (MB)
    \textbf{73.7} & % 1% gap accuracy
    > 4h & % 1% gap runtime (s)
    64 \\ % 1% gap memory (MB)

    Bidding &
    \textbf{100} & % optimal accuracy
    < 0.1 & % optimal runtime (s)
    10 & % optimal memory (MB)
    99.9 & % 1% gap accuracy
    < 0.1 & % 1% gap runtime (s)
    9 \\ % 1% gap memory (MB)

    Eeg &
    \textbf{73.5} & % optimal accuracy
    > 4h & % optimal runtime (s)
    33 & % optimal memory (MB)
    \textbf{73.5} & % 1% gap accuracy
    > 4h & % 1% gap runtime (s)
    34 \\ % 1% gap memory (MB)

    Fault &
    \textbf{78.1} & % optimal accuracy
    > 4h & % optimal runtime (s)
    15 & % optimal memory (MB)
    78.0 & % 1% gap accuracy
    > 4h & % 1% gap runtime (s)
    15 \\ % 1% gap memory (MB)

    Htru &
    \textbf{98.6} & % optimal accuracy
    > 4h & % optimal runtime (s)
    57 & % optimal memory (MB)
    98.5 & % 1% gap accuracy
    > 4h & % 1% gap runtime (s)
    72 \\ % 1% gap memory (MB)

    Magic &
    \textbf{74.4} & % optimal accuracy
    > 4h & % optimal runtime (s)
    45 & % optimal memory (MB)
    74.3 & % 1% gap accuracy
    > 4h & % 1% gap runtime (s)
    47 \\ % 1% gap memory (MB)

    Occupancy &
    \textbf{100} & % optimal accuracy
    34 & % optimal runtime (s)
    33 & % optimal memory (MB)
    99.8 & % 1% gap accuracy
    0.5 & % 1% gap runtime (s)
    11 \\ % 1% gap memory (MB)

    Page &
    98.0 & % optimal accuracy
    > 4h & % optimal runtime (s)
    161 & % optimal memory (MB)
    \textbf{98.3} & % 1% gap accuracy
    > 4h & % 1% gap runtime (s)
    169 \\ % 1% gap memory (MB)

    Raisin &
    \textbf{99.6} & % optimal accuracy
    180 & % optimal runtime (s)
    60 & % optimal memory (MB)
    99.4 & % 1% gap accuracy
    8 & % 1% gap runtime (s)
    7\\ % 1% gap memory (MB)

    Rice &
    \textbf{96.8} & % optimal accuracy
    > 4h & % optimal runtime (s)
    320 & % optimal memory (MB)
    96.7 & % 1% gap accuracy
    > 4h & % 1% gap runtime (s)
    201 \\ % 1% gap memory (MB)

    Room &
    \textbf{100} & % optimal accuracy
    2 & % optimal runtime (s)
    20 & % optimal memory (MB)
    \textbf{100} & % 1% gap accuracy
    0.2 & % 1% gap runtime (s)
    20 \\ % 1% gap memory (MB)

    Segment &
    \textbf{100} & % optimal accuracy
    483 & % optimal runtime (s)
    14 & % optimal memory (MB)
    99.8 & % 1% gap accuracy
    154 & % 1% gap runtime (s)
    10 \\ % 1% gap memory (MB)

    Skin &
    \textbf{99.9} & % optimal accuracy
    > 4h & % optimal runtime (s)
    20105 & % optimal memory (MB)
    99.8 & % 1% gap accuracy
    22 & % 1% gap runtime (s)
    176 \\ % 1% gap memory (MB)

    Wilt &
    \textbf{100} & % optimal accuracy
    < 0.1 & % optimal runtime (s)
    7 & % optimal memory (MB)
    99.9 & % 1% gap accuracy
    < 0.1 & % 1% gap runtime (s)
    7 \\ % 1% gap memory (MB)

    \bottomrule
    \end{tabular}
    \caption{ConTree results for maximum depth six. Runtime is averaged over five runs. The maximum memory usage observed is stated. If ConTree exceeds the time limit we report the best solution found at time-out. Best accuracy results are put in bold.}
    \label{tab:contree_d6}
\end{table*}

\begin{table*}[t!]
    \centering
    \begin{tabular}{l ccc cc c}
    \toprule
    &
    \multicolumn{3}{c}{Depth = 2} &
    \multicolumn{2}{c}{Depth = 3} &
    Depth = 4 \\
    \cmidrule(lr){2-4}
    \cmidrule(lr){5-6}
    \cmidrule(lr){7-7}
    
    Dataset &
    SAT-Shati &
    Quant-BnB &
    ConTree &
    Quant-BnB &
    ConTree & 
    ConTree \\ \midrule

    Avila &
    944 & % SAT-Shati d=2
    517 & % Quant-BnB d=2
    11 & % ConTree   d=2
    4319 & % Quant-BnB d=3
    13 & % ConTree   d=3
    20 \\ % ConTree  d=4

    Bank &
    208 & % SAT-Shati d=2
    477 & % Quant-BnB d=2
    4 & % ConTree   d=2
    516 & % Quant-BnB d=3
    4 & % ConTree   d=3
    4 \\ % ConTree  d=4

    Bean &
    934 & % SAT-Shati d=2
    532 & % Quant-BnB d=2
    15 & % ConTree   d=2
    1220 & % Quant-BnB d=3
    19 & % ConTree   d=3
    26 \\ % ConTree  d=4

    Bidding &
    446 & % SAT-Shati d=2
    495 & % Quant-BnB d=2
    7 & % ConTree   d=2
    579 & % Quant-BnB d=3
    8 & % ConTree   d=3
    10 \\ % ConTree  d=4

    Eeg &
    1466 & % SAT-Shati d=2
    509 & % Quant-BnB d=2
    14 & % ConTree   d=2
    14850 & % Quant-BnB d=3
    18 & % ConTree   d=3
    25 \\ % ConTree  d=4

    Fault &
    754 & % SAT-Shati d=2
    490 & % Quant-BnB d=2
    6 & % ConTree   d=2
    24750 & % Quant-BnB d=3
    7 & % ConTree   d=3
    10 \\ % ConTree  d=4

    Htru &
    1256 & % SAT-Shati d=2
    518 & % Quant-BnB d=2
    12 & % ConTree   d=2
    14781 & % Quant-BnB d=3
    15 & % ConTree   d=3
    24 \\ % ConTree  d=4

    Magic &
    1530 & % SAT-Shati d=2
    533 & % Quant-BnB d=2
    15 & % ConTree   d=2
    1458 & % Quant-BnB d=3
    18 & % ConTree   d=3
    26 \\ % ConTree  d=4

    Occupancy &
    450 & % SAT-Shati d=2
    494 & % Quant-BnB d=2
    7 & % ConTree   d=2
    608 & % Quant-BnB d=3
    8 & % ConTree   d=3
    11 \\ % ConTree  d=4

    Page &
    490 & % SAT-Shati d=2
    493 & % Quant-BnB d=2
    7 & % ConTree   d=2
    1184 & % Quant-BnB d=3
    8 & % ConTree   d=3
    11 \\ % ConTree  d=4

    Raisin &
    283 & % SAT-Shati d=2
    489 & % Quant-BnB d=2
    4 & % ConTree   d=2
    550 & % Quant-BnB d=3
    4 & % ConTree   d=3
    5 \\ % ConTree  d=4

    Rice &
    816 & % SAT-Shati d=2
    494 & % Quant-BnB d=2
    5 & % ConTree   d=2
    1198 & % Quant-BnB d=3
    6 & % ConTree   d=3
    9 \\ % ConTree  d=4

    Room &
    693 & % SAT-Shati d=2
    504 & % Quant-BnB d=2
    12 & % ConTree   d=2
    646 & % Quant-BnB d=3
    14 & % ConTree   d=3
    18 \\ % ConTree  d=4

    Segment &
    306 & % SAT-Shati d=2
    503 & % Quant-BnB d=2
    6 & % ConTree   d=2
    595 & % Quant-BnB d=3
    6 & % ConTree   d=3
    8 \\ % ConTree  d=4

    Skin &
    3965 & % SAT-Shati d=2
    571 & % Quant-BnB d=2
    78 & % ConTree   d=2
    736 & % Quant-BnB d=3
    89 & % ConTree   d=3
    123 \\ % ConTree  d=4

    Wilt &
    254 & % SAT-Shati d=2
    494 & % Quant-BnB d=2
    6 & % ConTree   d=2
    564 & % Quant-BnB d=3
    6 & % ConTree   d=3
    7 \\ % ConTree  d=4

    \bottomrule
    \end{tabular}
    \caption{Maximum memory usage (MB) per method. ConTree uses significantly less memory than previous approaches.}
    \label{tab:memory}
\end{table*}

All datasets used in the experiments are from the UCI machine learning repository \cite{uci2017}. We apply the same data preprocessing techniques as \citet{mazumder2022odt_continuous_bnb}: all features that do not assist prediction are removed (e.g., timestamps or unique identifiers).
For the experiments that measure the runtime and memory usage, we use the train split from their repository.\footnote{\url{https://github.com/mengxianglgal/Quant-BnB}} For the out-of-sample experiments in Appendix~\ref{app:oos}, we use the full dataset.
The preprocessing per dataset is as follows:

\begin{description}[font=\normalfont]
    \item[\textbf{Avila} \citep{misc_avila_459}] All features are kept as is.
    \item[\textbf{Bank} \citep{misc_banknote_authentication_267}] We keep all the features from the \emph{Banknote Authentication} dataset.
    \item[\textbf{Bean} \citep{bean_dataset}] We keep all the features from the \emph{Dry Bean} dataset.
    \item[\textbf{Bidding} \citep{misc_shill_bidding_dataset_562}] We remove the features \emph{Record\_ID}, \emph{Auction\_ID}, and \emph{Bidder\_ID} from the \emph{Shill Bidding} dataset.
    \item[\textbf{Eeg} \citep{misc_eeg_eye_state_264}] We keep all the features from the \emph{EEG Eye State} dataset.
    \item[\textbf{Fault} \citep{misc_steel_plates_faults_198}] We keep all the features from the \emph{Steel Plates Faults} dataset.
    \item[\textbf{Htru} \citep{htru_dataset}] We keep all the features from the \emph{HTRU2} dataset.
    \item[\textbf{Magic} \citep{misc_magic_gamma_telescope_159}] We keep all the features from the \emph{MAGIC Gamma Telescope} dataset.
    \item[\textbf{Occupancy} \citep{occupancy_dataset}] We remove the \emph{id} and \emph{date} features from the \emph{Occupancy Detection} dataset.
    \item[\textbf{Page} \citep{misc_page_blocks_classification_78}] We keep all the features from the \emph{Page Blocks Classification} dataset.
    \item[\textbf{Raisin} \citep{raisin_dataset}] All features are kept as is.
    \item[\textbf{Rice} \citep{rice_dataset}] We keep all the features from the \emph{Rice (Cammeo and Osmancik)} dataset.
    \item[\textbf{Room} \citep{room_dataset}] We remove the \emph{Date} and \emph{Time} features from the \emph{Room Occupancy Estimation} dataset.
    \item[\textbf{Segment} \citep{misc_image_segmentation_50}] We remove the \emph{REGION-PIXEL-COUNT} feature from the \emph{Image Segmentation} dataset because it has only one unique value.
    \item[\textbf{Skin} \citep{misc_skin_segmentation_229}] We keep all the features from the \emph{Skin Segmentation} dataset.
    \item[\textbf{Wilt} \citep{misc_wilt_285}] All features are kept as is.
\end{description}
\begin{table*}[t!]
\centering
\begin{tabular}{l ccc c cccccc }
\toprule

& & & & & \multicolumn{6}{c}{Optimal decision tree (given the binarization)} \\
\cmidrule(lr){6-11}

& & & &
CART & \multicolumn{2}{c}{10 quantiles} & \multicolumn{2}{c}{Guessing thresholds} &  \multicolumn{2}{c}{ConTree (All thresholds)} \\ 

\cmidrule(lr){5-5}
\cmidrule(lr){6-7}
\cmidrule(lr){8-9}
\cmidrule(lr){10-11}

Dataset &
$|\mathcal{D}|$ & $|\mathcal{F}|$ & $|\mathcal{Y}|$ &
Accuracy &
$|\mathcal{F}_{b}|$ & Accuracy & $|\mathcal{F}_{b}|$ & Accuracy & $|\mathcal{F}_{b}|$ & Accuracy \\ \midrule
                         
Avila          & 20867  & 10 & 12 & 
52.7 \small{$\pm$ 0.1} & % CART test accuracy
95 & % 10 Quantiles num binary features
56.6 \small{$\pm$ 0.1} & % 10 quantiles test accuracy
108 & % guessing num binary features
56.8 \small{$\pm$ 0.2} & % guessing test accuracy
41110 & % all binary features
\textbf{58.3 \small{$\pm$ 0.1}}  \\ % ConTree accuracy

Bank           & 1372   & 4  & 2  & 
92.7 \small{$\pm$ 0.9} & % CART test accuracy
40 & % 10 Quantiles num binary features
96.5 \small{$\pm$ 0.3} & % 10 quantiles test accuracy
29 & % guessing num binary features
96.0 \small{$\pm$ 0.3} & % guessing test accuracy
5016 & % all binary features
\textbf{97.0 \small{$\pm$ 0.3}}   \\ % ConTree accuracy

Bean           & 13611  & 16 & 7  & 
77.6 \small{$\pm$ 0.1} & % CART test accuracy
160 & % 10 Quantiles num binary features
84.6 \small{$\pm$ 0.2}  & % 10 quantiles test accuracy
194 & % guessing num binary features
79.9 \small{$\pm$ 0.5} & % guessing test accuracy
211343 & % all binary features
\textbf{86.8 \small{$\pm$ 0.1}}  \\ % ConTree accuracy

Bidding        & 6321   & 9  & 2  & 
98.4 \small{$\pm$ 0.1} & % CART test accuracy
61 & % 10 Quantiles num binary features
\textbf{99.3 \small{$\pm$ 0.1}} & % 10 quantiles test accuracy
23 & % guessing num binary features
\textbf{99.3 \small{$\pm$ 0.1}} & % guessing test accuracy
12527 & % all binary features
\textbf{99.3 \small{$\pm$ 0.1}} \\ % ConTree accuracy

Eeg            & 14980  & 14 & 2  & 
66.3 \small{$\pm$ 0.2} & % CART test accuracy
140 & % 10 Quantiles num binary features
69.7 \small{$\pm$ 0.2} & % 10 quantiles test accuracy
214 & % guessing num binary features
\textbf{70.7 \small{$\pm$ 0.2}} & % guessing test accuracy
5404 & % all binary features
70.6 \small{$\pm$ 0.3} \\ % ConTree accuracy

Fault          & 1941   & 27 & 7  & 
52.7 \small{$\pm$ 0.5} & % CART test accuracy
236 & % 10 Quantiles num binary features
65.1 \small{$\pm$ 0.6} & % 10 quantiles test accuracy
252 & % guessing num binary features
65.4 \small{$\pm$ 1.1} & % guessing test accuracy
19226 & % all binary features
\textbf{65.6 \small{$\pm$ 0.9}} \\ % ConTree accuracy

Htru           & 17898  & 8  & 2  & 
97.7 \small{$\pm$ 0.1} & % CART test accuracy
80 & % 10 Quantiles num binary features
97.7 \small{$\pm$ 0.1} & % 10 quantiles test accuracy
99 & % guessing num binary features
\textbf{97.8 \small{$\pm$ 0.1}} & % guessing test accuracy
123368& % all binary features
\textbf{97.8 \small{$\pm$ 0.1}} \\ % ConTree accuracy

Magic          & 19020  & 10 & 2  & 
79.0 \small{$\pm$ 0.3} & % CART test accuracy
100 & % 10 Quantiles num binary features
82.1 \small{$\pm$ 0.2} & % 10 quantiles test accuracy
198 & % guessing num binary features
\textbf{82.5 \small{$\pm$ 0.3}} & % guessing test accuracy
146815 & % all binary features
82.4 \small{$\pm$ 0.2} \\ % ConTree accuracy

Occupancy      & 20560   & 5  & 2  & 
98.9 \small{$\pm$ 0.1} & % CART test accuracy
45 & % 10 Quantiles num binary features
98.3 \small{$\pm$ 0.2} & % 10 quantiles test accuracy
31 & % guessing num binary features
98.9 \small{$\pm$ 0.1} & % guessing test accuracy
19709 & % all binary features
\textbf{99.0 \small{$\pm$ 0.1}} \\ % ConTree accuracy

Page           & 5473   & 10 & 5  & 
96.0 \small{$\pm$ 0.2} & % CART test accuracy
98 & % 10 Quantiles num binary features
94.8 \small{$\pm$ 0.3} & % 10 quantiles test accuracy
87 & % guessing num binary features
96.0 \small{$\pm$ 0.1} & % guessing test accuracy
9082 & % all binary features
\textbf{96.3 \small{$\pm$ 0.3}} \\ % ConTree accuracy

Raisin         & 900    & 7  & 2  &
84.8 \small{$\pm$ 0.7} & % CART test accuracy
70 & % 10 Quantiles num binary features
\textbf{85.7 \small{$\pm$ 1.2}} & % 10 quantiles test accuracy
62 & % guessing num binary features
84.6 \small{$\pm$ 0.6} & % guessing test accuracy
6289 & % all binary features
85.2 \small{$\pm$ 0.7} \\ % ConTree accuracy

Rice           & 3810   & 7  & 2  & 
93.2 \small{$\pm$ 0.5} & % CART test accuracy
70 & % 10 Quantiles num binary features
92.9 \small{$\pm$ 0.3} & % 10 quantiles test accuracy
81 & % guessing num binary features
\textbf{93.5 \small{$\pm$ 0.5}} & % guessing test accuracy
24635 & % all binary features
93.2 \small{$\pm$ 0.5} \\ % ConTree accuracy

Room           & 10129   & 16 & 4  & 
96.9 \small{$\pm$ 0.1} & % CART test accuracy
96 & % 10 Quantiles num binary features
98.1 \small{$\pm$ 0.2} & % 10 quantiles test accuracy
67 & % guessing num binary features
98.8 \small{$\pm$ 0.1} & % guessing test accuracy
3072 & % all binary features
\textbf{99.0 \small{$\pm$ 0.1}} \\ % ConTree accuracy

Segment        & 2310   & 18 & 7  & 
56.8 \small{$\pm$ 0.1} & % CART test accuracy
165 & % 10 Quantiles num binary features
85.2 \small{$\pm$ 0.6} & % 10 quantiles test accuracy
34 & % guessing num binary features
81.8 \small{$\pm$ 2.3} & % guessing test accuracy
12057 & % all binary features
\textbf{87.1 \small{$\pm$ 0.8}} \\ % ConTree accuracy

Skin           & 245057 & 3  & 3  & 
96.5 \small{$\pm$ 0.0} & % CART test accuracy
30 & % 10 Quantiles num binary features
96.1 \small{$\pm$ 0.0} & % 10 quantiles test accuracy
36 & % guessing num binary features
96.7 \small{$\pm$ 0.0} & % guessing test accuracy
765 & % all binary features
\textbf{96.8 \small{$\pm$ 0.0}} \\ % ConTree accuracy

Wilt           & 4839   & 5  & 5  & 
97.5 \small{$\pm$ 0.2} & % CART test accuracy
50 & % 10 Quantiles num binary features
97.7 \small{$\pm$ 0.2} & % 10 quantiles test accuracy
61 & % guessing num binary features
97.8 \small{$\pm$ 0.2} & % guessing test accuracy
22575 & % all binary features
\textbf{97.9 \small{$\pm$ 0.1}} \\ % ConTree accuracy
\midrule

Wins & & & &
0 & % CART wins
& 2 & % 10 quantiles ODT wins
& 5 & % guessing ODT wins
& 12 \\ % ConTree wins

\bottomrule
\end{tabular}
\caption{Out-of-sample accuracy and standard error~(\%) for trees of maximum depth three for five runs. ConTree (being able to split on all thresholds) significantly outperforms CART on the numeric data and optimal decision trees on data binarized using ten thresholds per feature or using a threshold guessing method \citep{mctavish2022sparse_trees_guess_ensembles}. The number of binary features considered are mentioned in the columns marked with $|\mathcal{F}_b|$. For ConTree, the column $|\mathcal{F}_b|$ mentions the number of binary features that would be required for an explicit binarization.}
\label{tab:oos}
\end{table*}

To determine all unique values per feature in a dataset, we sort the instances and check if consecutive values differ more than a small value $\varepsilon$. In this paper, we set $\varepsilon=1 \times 10^{-7}$.

\section{Results for Larger Depth Limits}
\label{app:large_depth}
Table~\ref{tab:contree_d5} and~\ref{tab:contree_d6} show the results for training optimal decision trees of depths five and six. The results are averaged over five runs. ConTree finds and proves optimal trees for eight and seven datasets respectively.

These tables also show that when an optimality gap of 1\% is allowed, the runtime is sometimes significantly reduced. There are also cases where the final training accuracy is higher when running with this optimality gap. This is because the permitted gap allows more of the search space to be pruned and less time is spent in checking every possible split. This motivates future work in improving the anytime performance of ConTree and in establishing an improved search order.

\section{Memory Usage Results}
\label{app:memory}

Table~\ref{tab:memory} reports the memory usage of the SAT-Shati, Quant-BnB, and ConTree for computing trees of depth two (all methods), three (Quant-BnB and ConTree), and four (only ConTree).
For depth two, memory poses no problem for all three methods.
For depth three, Quant-BnB uses 15GB of memory for Eeg and Htru and even 25GB for Fault. In contrast, ConTree's memory usage for these datasets is 18MB, 15MB, and 7MB respectively. At depth four, ConTree's maximum memory usage is for the Skin dataset, where it uses 123MB of memory. This highlights ConTree's memory efficiency in comparison to previous methods.

Table~\ref{tab:contree_d5} and Table~\ref{tab:contree_d6} further show the maximum memory usage when computing optimal trees of depths five and six. For all datasets, except Skin, the memory consumption stays within 320MB. The Skin dataset is the only dataset where ConTree uses a lot of memory when optimizing deeper trees. The reason for this is the size of the dataset which is ten times larger than the next largest dataset in our benchmark. If memory becomes an issue, ConTree's caching mechanism could be changed from dataset caching to branch caching to lower memory consumption \citep{demirovic2022murtree}.

\section{Out-of-Sample Results}
\label{app:oos}
To test the out-of-sample performance of ConTree, we compare it to the CART heuristic, and to optimal decision trees trained on binarized data.
We test two binarization approaches: one based on \emph{quantiles} and the other based on \emph{threshold guessing} using a reference ensemble \citep{mctavish2022sparse_trees_guess_ensembles}.
For the quantile approach, we binarize each of the numeric features using thresholds on ten quantiles of the feature distribution. For the threshold guessing approach, we follow \citet{mctavish2022sparse_trees_guess_ensembles} and train a gradient boosting classifier with the number of estimators set to ten times the number of continuous features and the maximum depth to two. After training, the least important features are removed iteratively until the performance of the ensemble drops.\footnote{\url{https://github.com/ubc-systopia/gosdt-guesses}.}

We repeat the experiment five times on five stratified 80\%-20\% train-test splits. For this, we use the full datasets rather than only the training splits used in previous experiments. In Table~\ref{tab:oos} we list the size of these datasets, the number of binary features $|\mathcal{F}_b|$ obtained by the two binarization methods, and the number of binary features one would need to search over all possible thresholds on the numeric features (what ConTree does, but without an explicit binarization).

Since all optimal methods obtain trees with the same accuracy, we only need to compare to one method on the binarized data. Because of its scalability, we train optimal decision trees for the binarized data using STreeD \citep{linden2023streed}.\footnote{\url{https://github.com/AlgTUDelft/pystreed}.}
For CART we use the $\operatorname{sklearn}$ implementation.
Each method is trained with a maximum depth of three. Within that depth limit, each method tunes a tree-size hyperparameter using five 80\%-20\% train-validation splits. CART uses cost-complexity tuning, STreeD tunes the depth and the number of nodes, and ConTree only tunes the depth.

Table~\ref{tab:oos} shows how ConTree significantly outperforms both CART on the continuous data and ODTs on binarized data. ConTree's test accuracy on average is 4.7\% higher than CART. For the Segment dataset, this difference is even 30\%. ConTree's test accuracy on average is 0.7\% higher than the trees optimized with the quantile binarization and 1.0\% higher than trees optimized with the threshold guessing binarization. A Wilcoxon signed rank test indicates that all these results are statistically significant~($p < 0.01$).